\documentclass{article}
\usepackage[preprint]{spconf}

\usepackage{tikz}
\usepackage{textcomp}

\usepackage{amsmath,graphicx,amsthm,color}
\usepackage{bm}
\usepackage{amsmath,amssymb,framed}
\usepackage{enumitem}
\usepackage{algorithm}
\usepackage{algpseudocode}
\algnewcommand\algorithmicforeach{\textbf{for each}}
\algdef{S}[FOR]{ForEach}[1]{\algorithmicforeach\ #1\ \algorithmicdo}

\newtheorem{proposition}{Proposition}
\newtheorem{theorem}{Theorem}

\graphicspath{ {./figure/} } 

\usepackage{soul}

\usepackage{microtype}      
\usepackage{xcolor}         

\newcommand\copyrighttext{%
  \footnotesize \textcopyright 2022 IEEE. Personal use of this material is permitted.
  Permission from IEEE must be obtained for all other uses, in any current or future
  media, including reprinting/republishing this material for advertising or promotional
  purposes, creating new collective works, for resale or redistribution to servers or
  lists, or reuse of any copyrighted component of this work in other works.}
\newcommand\copyrightnoticeee{%
\begin{tikzpicture}[remember picture,overlay]
\node[anchor=south,yshift=10pt] at (current page.south) {\fbox{\parbox{\dimexpr\textwidth-\fboxsep-\fboxrule\relax}{\copyrighttext}}};
\end{tikzpicture}%
}

\title{Robust and efficient change point detection using novel multivariate rank-energy GoF test}
\name{Shoaib Bin Masud and Shuchin Aeron}
\address{Dept. of ECE, Tufts University, Medford, MA, USA, 02155}
\begin{document}
%
\maketitle
\copyrightnoticeee
\begin{abstract}
 In this paper, we use and further develop upon a recently proposed multivariate, distribution-free Goodness-of-Fit (GoF) test  based on the theory of Optimal Transport (OT) called the Rank Energy ($\mathsf{RE}$) \cite{deb2021multivariate}, for non-parametric and unsupervised Change Point Detection (CPD) in multivariate time series data. We show that directly using $\mathsf{RE}$ leads to high sensitivity to very small changes in distributions (causing high false alarms) and it requires large sample complexity and huge computational cost. To alleviate these drawbacks, we propose a new GoF test statistic called as soft-Rank Energy ($\mathsf{sRE}$) that is based on entropy regularized OT and employ it towards CPD. We discuss the advantages of using $\mathsf{sRE}$ over $\mathsf{RE}$ and demonstrate that the proposed $\mathsf{sRE}$ based CPD outperforms all the existing methods in terms of Area Under the Curve (AUC) and F1-score on real and synthetic data sets. 
\end{abstract}
\section{Introduction}


A significant part of the multivariate time series analysis deals with the detection of unknown abrupt changes in a temporal signal that represent the transitions from one state to another. This problem, commonly referred to as the change point detection (CPD), has been extensively studied in statistics and machine learning and is found in many real-world applications including the analysis of physiological \cite{qi2014novel}, financial \cite{carvalho2007simulation} and sensor data \cite{he2006nonparametric}. 

Statistical CPD methods can be categorized into different criteria e.g., univariate vs. multivariate, parametric vs. nonparametric. Parametric or model-based methods assume that parameters of the underlying time series data distribution are either known or can be learned from data \cite{chamroukhi2013joint, lee2018time}. Parametric methods are advantageous if either of these assumptions holds true. When the distributions are unknown or learning the parameters from data is difficult, nonparametric methods are desirable. Approximation of divergence based on direct density-ratio estimation \cite{kanamori2009least} and the integral probability metric e.g., maximum mean discrepancy (MMD) \cite{gretton2012kernel} based two-sample multivariate Goodness-of-fit (GoF) test have been proposed in \cite{li2015scan, liu2013change} as the nonparametric ways to detect change points.
Some other nonparametric GoF tests that have been used for CPD such as Kolmogorov-Smirnov (KS), and Cramer-von-Mises \cite{hawkins2010nonparametric}, are univariate in nature. Recently \cite{cheng2020optimal} used the univariate Wasserstein two-sample test (W2T) based on the theory of OT in one-dimension \cite{ramdas2017wasserstein} for CPD. However, one of the drawbacks of this method is that it projects the data onto several one dimensional directions  and uses the average statistic, which can lead to the loss of detection power.

\textbf{\underline{Main contributions}} - In this work, we propose CPD using recently developed statistics known as the \textit{Rank Energy} ({\small $\mathsf{RE}$}) \cite{deb2021multivariate}. What makes {\small $\mathsf{RE}$} attractive for CPD is that it is a \textit{multivariate rank-based and distribution-free} GoF test \cite{deb2021multivariate}, where the notion of rank is derived leveraging the theory of Optimal Transport (OT) in high dimensions.

However, as we outline in more detail subsequently and as borne out from simulation results, directly using the sample version of {\small $\mathsf{RE}$} for CPD has some drawbacks, namely, high sensitivity to small changes (leading to high false alarm rates), high computational complexity, and large sample complexity. To alleviate these  shortcomings, we propose a new statistic, called \textit{soft-Rank Energy} ({\small $\mathsf{sRE}$}) that leverages the computational and sample efficient entropy regularized OT \cite{peyre2019computational} and exploit it for CPD. We demonstrate the advantages of using {\small $\mathsf{sRE}$} over {\small $\mathsf{RE}$}. We also evaluate the performances of {\small $\mathsf{RE}$} and {\small $\mathsf{sRE}$} on both toy and real datasets and compare them with the existing state-of-the-art (SOTA) methods. \\ 

The rest of the paper is organized as follows- In section \ref{sec:background} we provide necessary background on the multivariate {\small $\mathsf{RE}$} and highlight the pros and cons of using {\small $\mathsf{RE}$} in CPD. In section \ref{sec:proposed}, we then introduce the sample version of {\small $\mathsf{sRE}$} and employ it in CPD. In section \ref{sec:evaluation} we show improved AUC and F1-score for CPD on real datasets compared to state-of-the-art.

\section{Problem set-up}

\textbf{Notation}: We use bold-math capital letters $\bm{X}$ for multivariate random variables, bold-face capital letters $\mathbf{X}$ for matrices and maps, lower-case bold face bold-math $\bm{x}$ for vectors. We denote by $\mathcal{P}(\mathbb{R}^d)$ the set of probability measures on $\mathbb{R}^d$. The rest of the notation is standard and should be clear from the context. \\
Given a time series $\bm Z[t] \in \mathbb R^d$, $t= 1, 2, \dots$, where the data consists of distinct segments  $[0,\! \tau_1], \!\; [\tau_1\! +\!1, \!\tau_2], \dots, [\tau_{k-1} + 1, \tau_{k}]$ with $\tau_1 < \tau_2 < \dots$ Samples within each segment, $\bm X[t], t\in[\tau_{i-1}+1, \tau_i]$, are assumed to be i.i.d. and originated from an unknown distribution. In general, the distributions in two adjoining segments are considered to be different, whereas two distant segments may have a similar distribution. The primary objective of a nonparametric CPD method is to detect the change points $\tau_1, \tau_2, \dots, \tau_k$ without any prior knowledge or assumptions on the set of underlying distributions of the distinct time segments.

\textbf{A sliding window two-sample GoF test}: A common framework to detect change points is the \textit{sliding-window} approach \cite{li2015scan}. Given a window size $n$ on each side of a possible change point, an offline, unsupervised sliding window-based CPD method generally takes two adjacent time segments {\small $\{\bm Z[t -n], \bm Z[t -n +1], \dots, \bm Z[t -1]\}\sim \mu_X$} and {\small $\{\bm Z[t + 1], \bm Z[t + 2], \dots, \bm Z[n + t - 1]\}\sim \mu_Y$} and carry out a two-sample GoF test at each $t=1, 2, \dots$ 




For CPD, using the sliding window GoF test, a detection range $\delta$ is utilized. A time point $t$ is declared as a change point if the statistic $\sigma(t)$ at this point is the local maximum within $\delta$ and above a threshold $\eta$. In general, $\eta$ is specific to the statistical tests which is calculated from the distribution of the statistics under the null given a confidence level. The sliding-window based CPD procedure is described in Algorithm 1. In this context, our main contributions in this paper are to apply the recently proposed Rank-Energy {\small ($\mathsf{RE}$)} \cite{deb2021multivariate} as the GoF test for CPD, highlight its main properties and shortcomings and then propose a new test that improves upon this GoF test. 

\begin{algorithm} \caption{Sliding-window based CPD employing GoF test}
\label{alg1}
\begin{algorithmic}[1]
\Require $\bm Z[t]$: data, window size: $n$, threshold: $\eta$.
\ForEach {$t: n :(T-n)$}
\State {\small $\mathbf X[t] = \{\bm Z[t-n], \dots, \bm Z[t-1]\}$}
\State {\small $\mathbf Y[t] = \{\bm Z[t], \dots, \bm Z[t-n + 1]\}$}
\State {\small $\sigma(t) \gets \text{GoF-statistic}\big(\mathbf X[t], \mathbf Y[t]$}\big)
\EndFor
\\
Output: $\{\tau_1,\tau_2, \dots\}\!\!\gets \!\!\{t | \text{max}(\sigma(t)\!>\!\eta)\}$
\end{algorithmic}
\end{algorithm}
\vspace{-6mm}


\section{Background: Optimal Transport (OT) Based Multivariate Rank Energy Test}
\label{sec:background}

Optimal Transport, in it's most well-studied setting \cite{OTAM,peyre2019computational}, aims to find $ \mathbf{T}: \mathbb R^d \rightarrow \mathbb R^d$- a map that pushes a source distribution $\mu\in\mathcal P(\mathbb R^d)$ to a target distribution $\nu\in \mathcal P(\mathbb R^d)$ with a minimal expected squared Euclidean cost. That is, given two multivariate random variables, $\bm X\in \mathbb R^d \sim \mu$, and $\bm{Y}\in \mathbb R^d \sim \nu $, OT finds a map $\mathbf{T}$ that solves for,
{\small  
  \setlength{\abovedisplayskip}{2pt}
  \setlength{\belowdisplayskip}{\abovedisplayskip}
  \setlength{\abovedisplayshortskip}{0pt}
  \setlength{\belowdisplayshortskip}{3pt}
\begin{align}\label{monge}
   \inf_{\mathbf T} \int \|\bm x-\mathbf T(\bm x)\|^2 d\mu(\bm x) \;\; \text{subject to.}\;\; \bm{Y} = \mathbf{T}(\bm{x}) \sim \nu,
\end{align}}%
where $\|\cdot\|$ denotes the standard Euclidean norm in $\mathbb{R}^d$. Note that if $\mathbf T(\bm X) \sim \nu$ when $\bm X \sim \mu$, we write $\nu = \mathbf T_\#\mu$. In this case, the measure $\nu$ is referred to as the \textbf{\textit{push-forward}} of measure $\mu$ under the mapping $\mathbf{T}$. 

When $d=1$, it is known that the optimal map is $\mathbf{T} = \mathsf F_\nu^{-1} \circ \mathsf F_\mu$, where $\mathsf F_\mu$ and $\mathsf F_\nu$ are the (cumulative) distribution functions for $\mu$ and $\nu$, respectively \cite{OTAM, peyre2019computational}. If the target measure $\nu= \mathsf{U}[0, 1]$ is a uniform distribution on the line, then $\mathbf T = \mathsf F_\mu$, which is similar to the rank function in $1$-d. Extending this insight to the multivariate case, the notion of rank has been developed based on the following landmark result in OT theory.




\begin{theorem}[McCann~\cite{mccann1995existence}] \label{Theorem 1}
Assume $\mu, \nu \in \mathcal P(\mathbb R^d)$ be absolutely continuous measures, then there exist transport maps $\mathbf{R}(\cdot)$ and $\mathbf{Q}(\cdot)$, that are gradients of real-valued $d$-variate convex functions such that $\mathbf{R}_\# \mu =\nu, \;\; \mathbf{Q}_\#\nu = \mu$, $\mathbf{R}$ and $\mathbf{Q}$ are unique and $\mathbf{R}\circ \mathbf{Q}(\bm{x}) = \bm{x}$, $\mathbf{Q}\circ \mathbf{R}(\bm{y}) = \bm{y}$.
\end{theorem}
In particular, the fact that the gradients of convex functions are monotone maps \cite{mccann1995existence} has led the authors in \cite{deb2021multivariate, hallin2017distribution} to define $\mathbf{R}$ and $\mathbf{Q}$ as the multivariate rank and quantile map respectively under  appropriate selection of the target measure $\nu$. \\
Specific to this work, the authors in \cite{deb2021multivariate} use the uniform measure on the unit cube in $\mathbb R^d$ as the target measure and, developed the rank energy statistic as a GoF measure between distributions, whose sample version is stated below.\\

\textbf{Sample Multivariate Rank Energy \cite{deb2021multivariate}:} Given two sets of i.i.d. samples {\small $\{\!\bm X_1, \dots,\bm X_m \!\}\!\sim \!\mu_{X}\!\in\! \mathcal P(\mathbb R^d)$} and {\small $\{\bm Y_1, \dots,\bm Y_n\} \sim \mu_{Y}\in \mathcal P(\mathbb R^d)$}. Let {\small $\mu^{\bm X}_m \!\!=\!\!\frac{1}{m}\sum_{i=1}^m \delta_{\bm X_i}$}, {\small $\mu^{\bm Y}_n \!\!=\!\! \frac{1}{n}\sum_{i=1}^n\delta_{\bm Y_i}$} denote the empirical measures. 
A set of fixed Halton sequences \cite{chi2005optimal}, that mimic randomly chosen vectors in the unit cube in $\mathbb{R}^d$, denoted as {\small $\mathcal H_{m+n}^d:=\{\bm h_1^d, \dots, \bm h^d_{m+n}\}\subset [0, 1]^d$} with the empirical measure {\small$ \nu_{m,n}^{\mathbf{H}} = (m + n)^{-1}\sum_{ i = 1}^{ m + n} \delta_{\bm h_i}$} is taken as the target points. 

A joint empirical map $\mathbf {\widehat R}_{m,n}$ is computed between the joint empirical measure {\small $\mu_{m, n}^{\bm X, \bm Y}:=(m+n)^{-1}(m\mu^{\bm X}_m + n\mu^{\bm Y}_n)$} and $\nu_{m,n}^{\mathbf{H}}$ by solving the following discrete OT problem,
{\small  
  \setlength{\abovedisplayskip}{1.5pt}
  \setlength{\belowdisplayskip}{\abovedisplayskip}
  \setlength{\abovedisplayshortskip}{0pt}
  \setlength{\belowdisplayshortskip}{3pt}
\begin{align}
\label{eq:linear}
    \mathbf{\widehat P}= \arg \min_{\mathbf{P} \in \Pi} \sum_{i,j = 1}^{m+n} \mathbf{C}_{i,j} \mathbf{P}_{i,j},
\end{align}
}%
where {\small$\mathbf{C}_{i,j}$} is the squared Euclidean distance, and {\small $\Pi = \{ \mathbf{P}: \mathbf{P} \bm{1} = \frac{1}{m+n} \bm{1}, \bm{1}^\top \mathbf{P} = \frac{1}{m+n} \bm{1}^\top\}$}. The above formulation is also known as the Kantorovich relaxation \cite{peyre2019computational}. Now for any $\bm X_i$, one obtains a map as {\small $\widehat{\mathbf{R}}_{m,n}(\bm{X}_i) \!\!= \!\!\bm{h}_{\sigma(i)}$}, where $\sigma(i)$ is the non-zero index in the $i$-th row of $\widehat{\mathbf{P}}$. Given the ranks corresponding to $\bm X_i$'s and $\bm Y_j$'s, sample rank energy \cite{deb2021multivariate} is defined as, 
{\small  
  \setlength{\abovedisplayskip}{3pt}
  \setlength{\belowdisplayskip}{\abovedisplayskip}
  \setlength{\abovedisplayshortskip}{0pt}
  \setlength{\belowdisplayshortskip}{3pt}
\begin{align}\label{eqn:RE}
   \mathsf{RE}\!\! &:= \!\!\frac{2}{mn}\!\! \sum_{i, j=1}^{m, n} \!\!\|\widehat{\mathbf{R}}_{m,n}(\bm X_i)\! -\!\widehat{\mathbf{R}}_{m,n}(\bm Y_j)\| \!\! - \!\!\frac{1}{m^2}\!\!\!\sum_{i, j=1}^{m}\!\! \|\widehat{\mathbf{R}}_{m,n}(\bm X_i) \notag \\&- \widehat{\mathbf{R}}_{m,n}(\bm X_j)\| -
    \frac{1}{n^2} \sum_{i, j=1}^{n} \!\!\|\widehat{\mathbf{R}}_{m,n}(\bm Y_i)\!\! -\!\! \widehat{\mathbf{R}}_{m,n}(\bm Y_j)\|.
\end{align}
}

The null hypothesis $H_0$ is rejected if {\small $mn(m+n)^{-1}\mathsf{RE}$} is greater than the threshold and accepted otherwise. As shown in \cite{deb2021multivariate}, {\small $\mathsf{RE}$} is \textit{distribution-free under the null} for fixed sample size, a property that is desirable for selecting an optimal universal threshold to reject the null hypothesis, {\small $H_0:\mu_{\bm X}=\mu_{\bm Y}$}. 

We now note several \textbf{shortcomings of directly using $\mathsf{RE}$} for CPD. 


\vspace{-3mm}
\begin{itemize}[leftmargin=*]
\setlength {\itemsep}{-4pt} 

    \item \textbf{Sensitivity:} As shown from Figure \ref{toy_dataset}, we note that {\small $\mathsf{RE}$} is particularly sensitive to small shift in the mean and changes in the covariance. This characteristic may be useful in applications where it is required to detect any tiny changes. However, in many real-world datasets, these tiny changes may not be labeled as the true change points. Hence using {\small $\mathsf{RE}$} in CPD may lead to the detection of many false positives and deteriorate the overall performance.

    \item \textbf{Sample Complexity:} Curse of dimensionality -  In general case, sample complexity for reliable estimation of the sample rank map scales as $O(n^{-1/d})$ with dimension $d$ \cite{genevay2019sample}.
    
    \item \textbf{Computational complexity:} Being a linear program, the computational complexity of the OT plan for sample $\mathsf{RE}$ scales as $\mathcal O(n^3\log n)$, for a given sample size $n$, which is expensive.
\end{itemize}
\vspace{-2mm}
To alleviate these issues, in the next section, we introduce the sample soft-Rank Energy that leverages the properties of entropy regularized optimal transport \cite{peyre2019computational}.

\begin{figure*}[ht]
    \centering
    \includegraphics[width =\textwidth, height = 3cm]{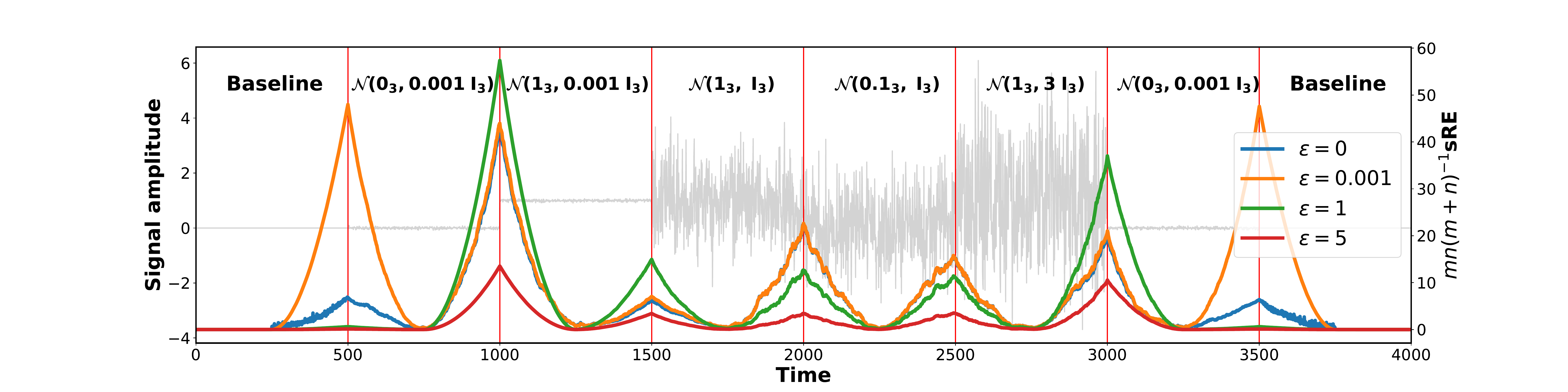}
    \vspace{-4mm}
    \caption{{\small $\mathsf{RE}$} and {\small $\mathsf{sRE}$} with $\varepsilon = \{0.001, 1, 5\}$ statistics (right axis) produced on a toy dataset using a \textit{sliding-window} CPD approach with a window size $n =250$. Dataset consists of 5 different $3$-dimensional Gaussian distributions $\mathcal N(\bm \mu_3, \sigma \mathbf{I}_3)$, with zero baselines on both ends. Here, $\mathbf I_3$ denotes the identity matrix. {\small $\mathsf{RE}(\varepsilon=0)$} and {\small $\mathsf{sRE}(\varepsilon=0.001)$} can detect the tiny changes between the baseline and $\mathcal N(0_3, 0.001\mathbf I_3)$ on both sides, whereas {\small $\mathsf{sRE}$} with {\small $\varepsilon=\{1, 5\}$} do not label these points as true change points.}
    \label{toy_dataset}
\end{figure*}
\vspace{-4mm}
\section{Proposed Sample Multivariate Soft Rank Energy}
\label{sec:proposed}
Given two sets of i.i.d. samples {\small $\{\bm X_1, \dots,\bm X_m \}\!\!\sim\!\! \mu_{X}\in \mathcal P(\mathbb R^d)$} and {\small $\{\bm Y_1, \dots,\bm Y_n\} \!\!\sim\!\! \mu_{Y}\in \mathcal P(\mathbb R^d)$}. To compute the soft rank, an entropy regularized OT  with a regularizer $\varepsilon$, is solved via Sinkhorn algorithm \cite{peyre2019computational} between the empirical joint source measure $\mu_{m, n}^{\bm X, \bm Y}$ and the reference measure $\nu_{m,n}^{\mathbf H}$, 
{\small  
  \setlength{\abovedisplayskip}{3pt}
  \setlength{\belowdisplayskip}{\abovedisplayskip}
  \setlength{\abovedisplayshortskip}{0pt}
  \setlength{\belowdisplayshortskip}{3pt}
\begin{align}
\label{empirical_sink}
    \mathbf{\widehat P}^\epsilon = \arg \min_{\mathbf{P} \in \Pi} \sum_{i,j = 1}^{m+n} \mathbf{C}_{i,j} \mathbf{P}_{i,j} - \varepsilon H(\mathbf{P}),
\end{align}
}%
where {\small$\mathbf{C}_{i,j}$} is the squared Euclidean distance, {\small $\varepsilon > 0$,  $\Pi = \{ \mathbf{P}: \mathbf{P} \bm{1} = \frac{1}{m+n} \bm{1}, \bm{1}^\top \mathbf{P} = \frac{1}{m+n} \bm{1}^\top\}$}, and {\small $H(\mathbf{P}) = - \sum_{i,j} \mathbf{P}_{i,j} \log \mathbf{P}_{i,j}$} is the entropy functional. $\mathbf {\widehat P^\varepsilon}$ is a diffused optimal plan, where the degree of diffusion increases as $\varepsilon$ increases. \\
Soft ranks are then obtained as follows. We compute a  \emph{row-normalized} plan $\mathbf{\bar P}^{\epsilon}$ via $\mathbf{\bar P}_{i,j}^\epsilon = \frac{\mathbf{\widehat P}_{i,j}^\varepsilon}{\sum_{j=1}^{m+n} \mathbf{\widehat P}_{i,j}^\varepsilon}$. Now, for any $\bm X_i$, one obtains the soft ranks via,
{\small  
  \setlength{\abovedisplayskip}{1pt}
  \setlength{\belowdisplayskip}{\abovedisplayskip}
  \setlength{\abovedisplayshortskip}{0pt}
  \setlength{\belowdisplayshortskip}{3pt}
\begin{align}\label{softrank}
    \widehat{\mathbf{R}}^{s,\epsilon} (\bm{X}_i) &= \sum_{j=1}^{m+n} \mathbf {\bar P^\varepsilon_{i, j}} \bm h_j = \mathbb E_{\mathbf {\widehat P^\varepsilon}}[\bm h_j|\bm X_i].
\end{align}
}%

In other words, soft ranks are the conditional expectation of Halton sequences under the joint distribution $\mathbf{\widehat P}^\varepsilon$ given the source samples. \\
Given the soft ranks corresponding to $\bm X_i$'s and $\bm Y_j$'s, sample soft rank energy is defined using the same formulation as in Equation \eqref{eqn:RE}:  
{\small  
  \setlength{\abovedisplayskip}{3pt}
  \setlength{\belowdisplayskip}{\abovedisplayskip}
  \setlength{\abovedisplayshortskip}{0pt}
  \setlength{\belowdisplayshortskip}{3pt}
\begin{align}
   \mathsf{sRE}\!\! &:= \!\!\frac{2}{mn}\!\! \sum_{i, j=1}^{m, n} \!\!\|\widehat{\mathbf{R}}^{s,\epsilon}(\bm X_i)\! -\!\widehat{\mathbf{R}}^{s,\epsilon}(\bm Y_j)\| \!\! - \!\!\frac{1}{m^2} \sum_{i, j=1}^{m} \|\widehat{\mathbf{R}}^{s,\epsilon}(\bm X_i) \notag \\&- \widehat{\mathbf{R}}^{s,\epsilon}(\bm X_j)\| -
    \frac{1}{n^2} \sum_{i, j=1}^{n} \|\widehat{\mathbf{R}}^{s,\epsilon}(\bm Y_i) - \widehat{\mathbf{R}}^{s,\epsilon}(\bm Y_j)\|.
\end{align}
}%
The null hypothesis $H_0$ is rejected if {\small $mn(m+n)^{-1}\mathsf{sRE}$} is greater than the threshold and accepted otherwise. We note the following result relating $\mathsf{sRE}$ and $\mathsf{RE}$.
\begin{proposition}\label{proposition1}
Soft rank energy $\mathsf{sRE}$ will converge to rank energy $\mathsf{ RE}$ as $\varepsilon\rightarrow 0$.
\end{proposition}
\begin{proof}
The unique minimizer $\mathbf{P^\varepsilon}$ of Equation \eqref{empirical_sink} converges to the optimal solution $\mathbf{P}$ (Equation \ref{eq:linear}) with a cost function $\mathbf C_{i,j}=\|\bm X_i-\bm h_j\|^2$ \cite{carlier2017convergence}, 
{\small  
  \setlength{\abovedisplayskip}{3pt}
  \setlength{\belowdisplayskip}{\abovedisplayskip}
  \setlength{\abovedisplayshortskip}{0pt}
  \setlength{\belowdisplayshortskip}{3pt}
\begin{align}\label{eq:plan_converge}
 \mathbf{P^\varepsilon}\rightharpoonup \mathbf{P},
\end{align}
}%
\textcolor{black}{where $\rightharpoonup$ denotes convergence w.r.t. weak topology.}
Let $\bar{\mathbf P}$ denote the \textit{row-normalized} $\mathbf P$. Then Equation \eqref{eq:plan_converge} implies that {\small $\lim_{\varepsilon\rightarrow 0}\mathbf{\bar P^\varepsilon}\rightharpoonup \mathbf{\bar P}$}, and {\small $\widehat{\mathbf{R}}^{s,\epsilon}(\bm X_i)\rightharpoonup\widehat{\mathbf{R}}_{m,n}(\bm{X}_i)$}. Therefore, {\small $\lim_{\varepsilon\rightarrow 0}\mathsf{sRE} = \mathsf{ RE}$}.

\end{proof}
We note the following \textbf{properties of $\mathsf{sRE}$ that help alleviate the issues in directly using $\mathsf{RE}$ for CPD}. 
\vspace{-1mm}
\begin{itemize}[leftmargin=*]
\setlength \itemsep{-4pt}
    \item Proposition \eqref{proposition1} implies that {\small$\mathsf{sRE}$} will be nearly distribution-free for small enough $\varepsilon$.
    
    \item \textbf{Sensitivity}: As shown in Figure \ref{toy_dataset}, {\small $\mathsf{sRE}$} is sensitive to small changes for $\varepsilon= 0.001$. For $\varepsilon = 1$, sensitivity decreases. However, {\small $\mathsf{sRE}$} still generates visible peaks at all the change points except the transition between the baseline and Gaussian distribution with tiny covariance. With $\varepsilon=5$, {\small $\mathsf{sRE}$} shows the least sensitivity with barely visible peaks at the change points. The entropic regularization parameter thus allows control of the degree of sensitivity that can be adapted or adjusted to control the false alarms. A good choice for CPD will be the  $\varepsilon$, for which {\small $\mathsf{sRE}$} is neither too sensitive nor totally unresponsive to changes.
    

    
    \item \textbf{Sample complexity}: Under some mild assumptions, namely, sub-Gaussianity of the measures, the estimation of the entropic optimal transport does not suffer from the curse of dimensionality for a sufficiently large $\varepsilon$ \cite{genevay2019sample}.
    
    \item \textbf{Computational complexity}: For a sample size $n$, the computational complexity of entropic optimal transport is {\small $\mathcal O(\varepsilon^{-2} n^2 \log\, n\|\mathbf C\|_\infty^2)$} \cite{peyre2019computational}. The smaller the $\varepsilon$, the costlier the computation. 
\end{itemize}

\begin{table*}[ht]
\centering
\setlength\tabcolsep{1.5pt}
\begin{tabular}{c|ccc||ccc}
\hline
      &              & CP-AUC   &          &              & CP-F1    &          \\
      &{\small \textbf{HASC-PAC2016}} & {\small \textbf{HASC2011}} & {\small \textbf{Beedance}} & {\small \textbf{HASC-PAC2016}} & {\small \textbf{HASC2011}} & {\small \textbf{Beedance}}  \\ \hline \rule{0pt}{3ex}
\textbf{W2T (Rank-Quantile), $d=1$ \cite{cheng2020optimal}}   & 0.689        & 0.576    & 0.721    & 0.748        & \textbf{0.824}    & 0.742    \\
\textbf{M-stat (IPM-based), $d\geq 1$, \cite{li2015scan}}  & 0.658        & 0.585    & 0.727    & 0.713        & 0.770    & 0.725   \\
\textbf{RE (Rank-Rank) $d \geq 1$, \cite{deb2021multivariate}}    & 0.718       & 0.529    & 0.694    & 0.631       & 0.643    &  0.672       \\
\textbf{sRE (soft Rank-Rank [This paper], $d \geq 1$}   & \textbf{0.747}        & \textbf{0.670}    & \textbf{0.739}    & \textbf{0.785}       & 0.796    & \textbf{0.745}       \\ \hline
\end{tabular}
\caption{Comparison between the proposed method and related state of art in literature.}
\vspace{-5mm}
\label{tab:result}
\end{table*}
\vspace{-5mm}
\section{Numerical Experiments \& Simulations}
\label{sec:evaluation}


\textbf{\underline{Experimental Setup:}} We compare the performances of our methods with two other existing algorithms, the univariate distribution-free Wasserstein two-sample test (W2T) \cite{cheng2020optimal} based CPD and the multivariate M-Statistic (MStat) \cite{li2015scan} based CPD that uses Maximum-Mean Discrepancy (MMD) \cite{gretton2012kernel} for measuring GoF. The Gaussian kernel with unit variance is used to compute MStat. 

For a fair comparison, we apply the  optimal matched filter proposed in \cite{cheng2020optimal} on W2T and MStat that improves the performances of these methods significantly. It is to be noted that no smoothing filter was applied on {\small $\mathsf{RE}$} and {\small $\mathsf{sRE}$} statistics. 

The hyperparameters we use in the CPD algorithm are the entropic regularizer $\varepsilon$, and the detection threshold $\eta$ to compute the F1 score. To compare the methods on an equal footing, we use the same window size $n$ and detection range $\delta$ for all the methods. The optimal $\eta$, $\varepsilon$ selected for the proposed methods, window size $n$, and the detection range $\delta$ along with the specifications of the used datasets can be found in Table \ref{tab:real_data}. It is worthwhile to note that, since the Beedance dataset has a comparatively shorter sequence, we padded it with a zero sequence of length $n$ on both ends.
\vspace{-3mm}
\begin{table}[ht]
\centering
\setlength\tabcolsep{1.2pt}
\begin{tabular}{cccc}
                  &{\small \textbf{HASC2011}\cite{ichino2016hasc}} & {\small \textbf{HASC-PAC2016}\cite{ichino2016hasc}} & {\small \textbf{Beedance}\cite{oh2008learning}} \\\hline
{\small domain}&$\mathbb R^3$&$\mathbb R^3$&$\mathbb R^3$\\
{\small $\#$ subjects} & 2                 & 10                    & 6                 \\
{\small $\#$actions}&6                &6                      & 3                  \\
{\small$\#$ CP}          & 65                & 13                    & 19                \\
$n$               & 500               & 500                   & 50                \\
$\varepsilon$     &2                  &1                      &1\\
$\delta$          & 250               & 250                   & 10 \\
$\eta$&  0.52 & 0.52 & 0.25
\end{tabular}
\caption{{\small $\#$} denotes the total number and CP is for change points.}
\vspace{-3mm}
\label{tab:real_data}
\end{table}

\textbf{\underline{Result on real data}}: 
Table \ref{tab:result} demonstrates the performance comparison of the proposed methods with MStat \cite{li2015scan} and W2T \cite{cheng2020optimal}. The proposed methods demonstrate robust results for CPD under the AUC metric. {\small $\mathsf{RE}$} gains higher AUC compared to W2T and MStat on HASC-PAC2016 dataset but fails to outperform on HASC2011 and Beedance datasets. {\small$\mathsf{RE}$}'s highly sensitive nature to tiny changes brings about many false alarms, thus explains the lower AUC for these datasets. On the other hand, {\small$\mathsf{sRE}$} outperforms all the methods on all three datasets in terms of AUC score. To be noted here, we observe significant improvement of the AUC using W2T and MStat on the Beedance dataset after the inclusion of zero-padding on both ends.

Though under the AUC metric, {\small$\mathsf{RE}$} shows a comparable result, we observe lower F1-scores compared to W2T and MStat on all three datasets. Since we did not apply any filter to smooth out the {\small$\mathsf{RE}$} statistics, several spurious maxima exist outside of the detection range $\delta$ on both sides of the peaks. Moreover, {\small$\mathsf{RE}$} also produces a lot of false alarms because of its higher sensitivity to small changes. As a result, {\small$\mathsf{RE}$} achieves slightly lower F1-scores on all three datasets. On the other hand, {\small$\mathsf{sRE}$} achieves either higher or comparable F1-scores in all three datasets. On HASC-PAC2016 and Beedance datasets, {\small$\mathsf{sRE}$} achieves the highest F1-score. W2T-based CPD achieves the maximum F1-score on HASC2011, whereas {\small$\mathsf{sRE}$}  performs comparably.

We also compare the performance of {\small $\mathsf{sRE}$} to the method called \textbf{KL-CPD} \cite{chang2019kernel}, which is a \textit{semi-supervised} CPD method. The best AUC is achieved by \textbf{KL-CPD} which is \textit{a kernel-based semi-supervised method trained by a deep generative model}. \textbf{KL-CPD} achieves AUC of \textbf{0.677} and \textbf{0.649} on the Beedance and HASC2011 dataset, respectively, which is clearly lower than the AUC scores obtained by the proposed {\small $\mathsf{sRE}$}.



\section{Conclusion and Future work}
\label{sec:foot}
In this paper, we employ recently developed multivariate GoF statistics to detect change points in an unsupervised, offline approach. We also propose a new statistic that depends on a regularization parameter which allows control of the degree of sensitivity. With an appropriate regularizer, we have shown that our proposed statistic lowers the false positive rate, hence outperforms state of the art in CPD under the AUC and F1-score metric. Future work will investigate theoretical properties of $\mathsf{sRE}$ and explain the smoothing effect in CPD as a function of the entropic regularization.

\clearpage
\section{Acknowledgement}
This research was sponsored by the U.S. Army DEVCOM Soldier Center, and was accomplished under Cooperative Agreement Number W911QY-19-2-0003. The views and conclusions contained in this document are those of the authors and should not be interpreted as representing the official policies, either expressed or implied, of the U.S. Army DEVCOM Soldier Center, or the U.S. Government. The U. S. Government is authorized to reproduce and distribute reprints for Government purposes notwithstanding any copyright notation hereon.

We also acknowledge support from the U.S. National Science Foundation under award HDR-1934553 for the Tufts T-TRIPODS Institute.  Shuchin Aeron is also supported in part by NSF CCF:1553075, NSF RAISE 1931978, NSF ERC planning 1937057, and AFOSR FA9550-18-1-0465.
\bibliographystyle{IEEEbib}
\bibliography{ref}

\end{document}